\documentclass[letterpaper,10pt, conference]{_Aux/ieeeconf}

\IEEEoverridecommandlockouts
\usepackage{times}
\usepackage{setspace}
\usepackage{url}
\spacing{1}
\usepackage[utf8]{inputenc}
\usepackage[T1]{fontenc}
\usepackage{graphicx}		
\usepackage{wrapfig}
\usepackage[format=plain,font=footnotesize,labelfont=bf,labelsep=period]{caption}
\usepackage{sidecap} 
\usepackage[export]{adjustbox}
\usepackage{subcaption}
\usepackage[font=small]{caption}
\usepackage{float}

\usepackage{comment}

\usepackage{amsmath} 
\usepackage{amssymb}  
\usepackage{amsthm}
\usepackage{mathtools}
\usepackage[normalem]{ulem}
\usepackage{paralist}	
\usepackage[space]{grffile} 
\usepackage{color}
\let\labelindent\relax
\usepackage{enumitem}
\usepackage{bm}
\usepackage{cancel}
\usepackage{hhline}
\usepackage[c2 , nocomma]{optidef}
\usepackage{oubraces}
\usepackage{algorithmic}
\usepackage{graphicx}
\usepackage{textcomp}

\usepackage{moreverb, url}

\usepackage[ruled,vlined]{algorithm2e}
\usepackage{stackengine}
\usepackage{dirtytalk}

\usepackage{array}
\usepackage{verbatim}
\usepackage{upgreek}

\usepackage{amsfonts}
\usepackage{lipsum}
\usepackage{cite}
\usepackage{hyperref}

\newtheorem{theorem}{Theorem}

\theoremstyle{definition}
\newtheorem{definition}{Definition}
\theoremstyle{remark}
\newtheorem{remark}{Remark}
\theoremstyle{definition}
\theoremstyle{definition}

\newtheorem{example}{Example}

\DeclareMathOperator*{\eig}{eig}

\DeclareMathOperator*{\argmax}{arg\,max}
\DeclareMathOperator*{\argmin}{arg\,min}

\definecolor{blue}{RGB}{38,38,134}
\definecolor{darkblue}{RGB}{0,0,102}
\definecolor{lightblue}{RGB}{77,77,148}

\definecolor{gold}{RGB}{234, 170, 0}
\definecolor{metallic_gold}{RGB}{139, 111, 78}

\DeclareMathOperator{\sgn}{sgn}

\DeclareMathOperator{\rank}{rank}

\usepackage{xfrac}

\usepackage{dblfloatfix}

\usepackage{pdfpages}
\usepackage{graphicx}

\usepackage{enumitem}

\begin{document}

\title{ \bf
Bayesian Optimal Experimental Design for Robot Kinematic Calibration
}

\author{Ersin Da\c{s}$^{1}$, Thomas Touma$^{1}$, Joel W. Burdick$^{1}$
\thanks{*This work was supported by NASA Grant 80NSSC21K1032.}
\thanks{$^{1}$E. Da\c{s}, T. Touma, J. W. Burdick are with the Department of Mechanical and Civil Engineering, California Institute of Technology, Pasadena, CA 91125, USA. ${\tt\small \{ersindas, ttouma, jburdick \}@caltech.edu}$ } }

\maketitle

\begin{abstract}
This paper develops a Bayesian optimal experimental design for robot kinematic calibration on ${\mathbb{S}^3 \!\times\! \mathbb{R}^3}$. Our method builds upon a Gaussian process approach that incorporates a geometry-aware kernel based on Riemannian Mat\'ern kernels over ${\mathbb{S}^3}$. To learn the forward kinematics errors via Bayesian optimization with a Gaussian process, we define a geodesic distance-based objective function. Pointwise values of this function are sampled via noisy measurements taken using fiducial markers on the end-effector using a camera and computed pose with the nominal kinematics. The corrected Denavit-Hartenberg parameters are obtained using an efficient quadratic program that operates on the collected data sets. The effectiveness of the proposed method is demonstrated via simulations and calibration experiments on NASA's ocean world lander autonomy testbed (OWLAT). 
\end{abstract}

\section{Introduction} 
Robot control methods generally rely on {\em forward kinematic functions} (which compute the end-effector pose as a function of joint angles), and {\em inverse kinematic functions} (which compute the joint variables that position the manipulator end-effector at a desired pose) \cite{murray2017}. Hence, high-precision \textit{kinematic calibration} of robotic systems—through the accurate identification of manipulator geometry parameters and joint variables—is essential for precise manipulation \cite{hollerbach2016model, wu2015geometric, hollerbach1989survey, khalil2002modeling, nubiola2013absolute}. This becomes particularly challenging in less structured environments, which require advanced, data-efficient calibration techniques. 

For example, robotic surface science and sampling missions have been proposed for the icy moons of Europa and Enceladus \cite{hand2022}, which may host subsurface oceans. Cryovolcanic activity might transport water/ice from a subsurface ocean through surface cracks. A robotic arm could gather surface material samples from the vicinity of these cracks, which can then be analyzed in situ for molecules that indicate the presence of life in the subsurface oceans. Kinematic accuracy of the sampling robot arm is crucial for these missions since proper science interpretation of the samples requires accurate localization of their origins \cite{bowkett2021, wagner2024, chien2024}. Furthermore, given the constraints of short-duration, battery-powered missions with round-trip communication delays, recalibration must be automatic, sample-efficient, and minimally dependent on metrology resources.

Similar to other techniques for system identification \cite{ljung1998system}, the selection of informative pose configurations for designing the input dataset in kinematic calibration is crucial. Various optimization methods for optimal sets of pose measurements often rely on observability indexes \cite{sunactive, joubair2013, kamali2019, wang2017finding}. These methods typically involve generating a large set of randomly selected candidate configurations, followed by selecting the most optimal configurations from this set. However, actual or measured poses differ from the computed configurations due to kinematic errors, which results in robustness issues \cite{sun2008, wang2017finding}. Therefore, the calibration experimental design process should be adaptive (the choice of the next robot pose to sample should be based on previous measurements), online, and have convergence guarantees \cite{caverly2025}. 

Our previous study \cite{dacs2023active} introduced a Gaussian Process (GP)-based non-parametric calibration framework that addressed these challenges. Rather than identifying a corrected set of Denavit-Hartenberg (DH) kinematic parameters, we used a set of GP models to represent the residual kinematic error of the arm over the workspace. These residual errors were modeled with seven independent GPs modeling each orientation and translation direction. However, Gaussian models are not really suited for unit quaternion-based orientation representations, where a Bingham distribution \cite{bingham1974} properly models a probability distribution on the unit sphere ${\mathbb{S}^{n-1} \!\subset\! \mathbb{R}^{n}}$. This prior study also proposed a calibration experiment design method to minimize the number of samples. However, the specific geometry of ${\mathbb{S}^3 \!\times\! \mathbb{R}^3}$ is not considered, and classical kernel functions are adopted from Euclidean space-based kernel classes, which are not valid kernels for $\mathbb{S}^3$. Furthermore, that study utilizes an experimental design strategy that averaged the seven GPs to optimize joint angles instead of the end-effector pose. 

This paper addresses the aforementioned limitations by introducing a novel geometry-aware online Bayesian optimization framework for robot kinematic calibration. We present a valid geometry-aware kernel function specifically designed for ${\mathbb{S}^3 \!\times\! \mathbb{R}^3}$ based on Riemannian Mat\'ern kernels proposed in \cite{borovitskiy2020matern}, ensuring that our GP models accurately capture the underlying non-Euclidean structure of the calibration problem. Furthermore, instead of optimizing over joint angles in the experimental design process, we focus on optimizing the end-effector poses in task space. This approach is more resilient to uncertainties in joint angles, enabling a more accurate and reliable calibration, particularly in complex scenarios where both rotational and translational motions must be precisely calibrated. We define a geodesic distance-based objective function for Bayesian optimization to learn the kinematics errors using GPs, and its pointwise values are sampled using fiducial markers on the robot arm. We validate the efficacy of the proposed approach through experiments on a robotic arm in NASA's ocean world lander autonomy testbed (OWLAT) \cite{nayar2021, tevere2024, thangeda2024}.

\section{Preliminaries} 
\textbf{Notation:} We use standard notation: ${\mathbb{R}}$ represents the set of real numbers, and $\mathbb{N}$ represents the set of natural numbers. The Euclidean norm of a matrix is denoted by $\|\!\cdot\!\|$. ${\langle \cdot, \cdot \rangle}$ is the dot product. A zero vector is denoted by ${\bf 0}$. ${\mathbf{f} \!\sim\! \mathbf{GP} (\mu, k)}$ denotes that function $\mathbf{f}(\cdot)$ is sampled from a GP (which is reviewed in section \ref{section:GP}). We use the symbol $\mathbf{f}$ to represent various functions distinguished by subscripts and context.

\textbf{Rigid Body Displacements:} We will use right-handed orthogonal Euclidean reference frames to denote rigid body poses. Reference frame orientations are described (relative to a fixed world reference frame) by a rotation matrix ${\mathbf{R} \!\in\! \mathbb{R}^{3 \!\times\! 3}}$ belonging to the \textit{special orthogonal group} denoted by ${ \mathrm{SO}(3) \!\triangleq\! \{ \mathbf{R} \!\in\! \mathbb{R}^{3\times3} \!:\! \mathbf{R}^\top \mathbf{R} \!=\! \mathbf{I}, \det(\mathbf{R}) \!=\! 1 \} }$. Reference frame origin locations are described by a vector ${\mathbf{p} \!\in\! \mathbb{R}^3 }$. Homogeneous matrices in the \textit{special Euclidean group} ${\mathrm{SE(3)} \!\equiv\! \mathrm{SO(3)} \!\times\! \mathbb{R}^3}$:  
\begin{equation*}
    { \mathrm{SE}(3) \triangleq \left \{ \mathbf{T} \triangleq  \begin{bmatrix}
        \mathbf{R} & \mathbf{p} \\
        \mathbf{0}^\top & 1 \end{bmatrix} \!\subset \mathbb{R}^{4\times 4} \bigg | \mathbf{R} \in \mathrm{SO}(3), \mathbf{p} \in \mathbb{R}^3 \right \} }
\end{equation*}
denote a combined rotation and translation.
Both ${ \mathrm{SO}(3)}$ and ${ \mathrm{SE}(3)}$ are \textit{matrix Lie groups}; therefore, their group operations are smooth. Moreover, the geometries of these manifolds play a role in our approach.

A quaternion $\mathbf{q}$ can be represented as $\mathbf{q} \!=\! w\mathbf{1} \!+\! x\mathbf{i} \!+\! y \mathbf{j} \!+\! z\mathbf{k} \!=\! \left(w, x, y, z\right)^\top$, where the three imaginary units $\mathbf{i}, \mathbf{j}, \mathbf{k}$ satisfy  the multiplication rules ${\mathbf{i}^{2} \!=\! \mathbf{j}^{2} \!=\! \mathbf{k}^{2} \!=\! \mathbf{i j k} \!=\! -\mathbf{1}}$ and ${\left [w~x~y~z\right]^\top \!\in\! \mathbb{R}^{4}}$. Quaternion $\mathbf{q}$ can also be represented by ${\mathbf{q} \!\triangleq\! [a, \mathbf{v}]}$, where ${a \!=\! w \!\in\! \mathbb{R}}$ is the {\em scalar} part and ${\mathbf{v} \!=\! \left(x, y, z\right)^\top \!\in\! \mathbb{R}^{3}}$ is the {\em vector} part. The conjugate $\overline{\mathbf{q}}$ of quaternion $\mathbf{q}$ is given by ${\overline{\mathbf{q}} \triangleq w \!-\! x \mathbf{i} \!-\! y \mathbf{j} \!-\! z \mathbf{k}}$. 

Spatial rotations can be represented with unit quaternions, ${\mathbf{q} \!\in\! \mathbb{H}_1}$ with ${1 \!=\! |\mathbf{q}| \!\triangleq\! \mathbf{q} \cdot \overline{\mathbf{q}}}$ and ${\mathbf{q}^{-1} \!=\! \overline{\mathbf{q}}}$. Unit quaternion representations support efficient numerical schemes and help to avoid singularities. The set of unit quaternions forms the surface of a 3-dimensional sphere, $\mathbb{S}^{3}$, which is a Lie group.

A unit quaternion can be converted to a rotation matrix using a mapping ${ \mathbf{f}_{\mathbf{q} \to \mathbf{R}} \!:\! \mathbb{S}^{3} \!\to\! \mathrm{SO}(3)}$:
\begin{equation*}
\mathbf{f}_{\mathbf{q} \to \mathbf{R}} (\mathbf{q}) = \mathbf{I} + 2w[\mathbf{v}]_\times + 2[\mathbf{v}]_\times [\mathbf{v}]_\times ,
\end{equation*}
where ${[\mathbf{v}]_\times}$ represents the skew-symmetric matrix corresponding to $\mathbf{v}$. 

One of the many methods to compute the quaternion corresponding to a given rotation matrix \cite{sarabandi2019} can be represented as a mapping ${ \mathbf{f}_{  \mathbf{R} \to \mathbf{q}} \!:\! \mathrm{SO}(3) \!\to\! \mathbb{S}^{3}  }$:
\begin{equation*}
\mathbf{f}_{  \mathbf{R} \to \mathbf{q}} (\mathbf{R}) \!=\!  
\dfrac{1}{2} \! \!
\begin{bmatrix}
\sqrt{1 \!+\! \mathbf{R}_{11} \!+\! \mathbf{R}_{22} \!+\! \mathbf{R}_{33}} \\
\sgn(\mathbf{R}_{32} \!-\! \mathbf{R}_{23}) \sqrt{1 \!+\! \mathbf{R}_{11} \!-\! \mathbf{R}_{22} \!-\! \mathbf{R}_{33}} \\
\sgn(\mathbf{R}_{13} \!-\! \mathbf{R}_{31}) \sqrt{1 \!-\! \mathbf{R}_{11} \!+\! \mathbf{R}_{22} \!-\! \mathbf{R}_{33}} \\
\sgn(\mathbf{R}_{21} \!-\! \mathbf{R}_{12}) \sqrt{1 \!-\! \mathbf{R}_{11} \!-\! \mathbf{R}_{22} \!+\! \mathbf{R}_{33}}
\end{bmatrix} \!\!,
\end{equation*}
where $\mathbf{R}_{ij}$ denotes the element in the $i$-th row and $j$-th column of the rotation matrix.

The similarity of two unit quaternions (which represent rotations), $\mathbf{q}_1$ and $\mathbf{q}_2$, can be measured by their {\em geodesic distance}, ${d_{\mathbb{R}^3} \!:\! \mathbb{S}^3 \!\times\! \mathbb{S}^3  \!\to\! \mathbb{R}_0^+}$,
\begin{equation}
\label{eq:quad_dist}
d_{\mathbb{S}^3}(\mathbf{q}_1, \mathbf{q}_2) = 2 \cos^{-1} \left( \left| \langle \mathbf{q}_1, \mathbf{q}_2 \rangle \right| \right),
\end{equation}
the shortest paths between the quaternions on the surface of $S^3$.  And, similar to straight lines in Euclidean space, the second derivative is zero everywhere along a geodesic. 

Note that in \eqref{eq:quad_dist}, ${\left| \langle \mathbf{q}_1, \mathbf{q}_2 \rangle \right| \!=\! \left| \langle \mathbf{q}_1, -\mathbf{q}_2 \rangle \right|}$; therefore, $\mathbf{q}_2$ and $-\mathbf{q}_2$ equivalently represent the same rotation. We thus select the shortest of the two arcs connecting $\mathbf{q}_1$ to $\mathbf{q}_2$ or $-\mathbf{q}_2$:
\begin{equation*}
 \min \!\left( 2 \cos^{-1} \!\left( \langle \mathbf{q}_1, \mathbf{q}_2 \rangle \right), 2 \cos^{-1} \! \left( \langle \mathbf{q}_1, -\mathbf{q}_2 \rangle \right) \right) \!\equiv\! d_{\mathbb{S}^3}(\mathbf{q}_1, \mathbf{q}_2).
\end{equation*}
Therefore, \eqref{eq:quad_dist} implicitly accounts for the ambiguous sign. 

\textbf{Forward Kinematics:}  This paper focuses on serial, or open, kinematic chain that consists of ${n_j \!\in\! \mathbb{N}}$ joints. Each joint has one degree of freedom, a revolute joint with the \textit{joint variable} ${\theta_i \!\in\! [0, 2 \pi ) }$, or a prismatic joint with the joint variable ${\theta_i \!\in\! \mathbb{R} }$, connecting ${n_l \!=\! n_j \!+\! 1}$ links that form a single serial chain. Given a set of joint angles ${\uptheta \! \triangleq \! [ \theta_1, \ldots, \theta_{n_j} ]^\top \!\in\! \Theta \!\subset\! \mathbb{R}^{n_j}}$, the {\em forward kinematic} function describes the location of the \textit{end effector} reference frame $\mathcal{T}$ relative to the \textit{base frame} $\mathcal{B}$;  ${ \mathbf{f}_{\uptheta \to \mathbf{T}} \!:\! \Theta \!\to\! \mathrm{SE(3)}}$. 

One of the most popular techniques for parameterizing the forward kinematics function is the \textit{Denavit-Hartenberg} (DH) convention.  In this convention (see \cite{murray2017} for details), {\em link frames} are defined for each link in the chain.  Let the transformation between adjacent link frames be denoted as ${ \mathbf{f}_{\theta_i \to \mathbf{T}_{i\!-\!1,i}} \!:\! \Theta_i \!\to\! \mathrm{SE(3)}}$, where ${\Theta_i \!\in\! \mathbb{S}^{1}}$ or ${\Theta_i \!\in\! \mathbb{R}}$. As a function of the link frame transformations, the forward kinematics map is given by 
\begin{equation}
\label{eq:forward}
\mathbf{f}_{\uptheta \to \mathbf{T}} \!=\!  \begin{bmatrix}
\mathbf{R}(\uptheta) & \mathbf{p}(\uptheta) \\
\mathbf{0}^\top & 1 \end{bmatrix} \!=\! \mathbf{f}_{\mathcal{B}} \Bigg ( \prod_{i=1}^{i = n_j} \mathbf{f}_{\theta_i \to \mathbf{T}_{i\!-\!1,i}} \Bigg )  \mathbf{f}_{\mathcal{T}} ,
\end{equation}
where ${\mathbf{f}_{\mathcal{B}}, \mathbf{f}_{\mathcal{T}} \!\in\! \mathrm{SE}(3)}$ are the transformation matrices for the base frame and end effector with respect to an inertial frame, and the last link, respectively. The link transformation matrices ${ \mathbf{f}_{\theta_i \!\to\! \mathbf{T}_{i\!-\!1,i}}}$ have the form
\begin{eqnarray*}
&\mathbf{f}_{\theta_i \!\to\! \mathbf{T}_{i\!-\!1,i}}(\theta_i) = \\
&\begin{bmatrix} 
\cos{\phi_i} & \!\!-\!\sin{\phi_i} \cos{\alpha_i} & \sin{\phi_i} \sin{\alpha_i} & a_i \! \cos{\phi_i}\\ 
\sin{\phi_i} & \cos{\phi_i} \cos{\alpha_i} & \!\!-\! \cos{\phi_i} \sin{\alpha_i} & a_i \! \sin{\phi_i}\\
0 & \sin{\alpha_i} &  \cos{\alpha_i} & d_i \\
0 & 0 &  0 & 1
\end{bmatrix} \!\! ,
\end{eqnarray*}
where ${\phi_i \!\in\! \mathbb{S}^{1}}$ is the \textit{joint angle offset}, ${\alpha_i \!\in\! \mathbb{S}^{1}}$ is the \textit{twist angle}, ${a_i \!\in\! \mathbb{R} }$ is the \textit{link length}, and ${d_i \!\in\! \mathbb{R} }$ is the \textit{link offset}. And, $\phi_i$ is the joint variable if the joint is revolute, while $d_i$ is the joint variable if the joint is prismatic \cite{murray2017}. 

\section{\texorpdfstring{Bayesian Optimization on ${\mathbb{S}^3 \!\times\! \mathbb{R}^3}$}{Bayesian Optimization on S3 x R3}}
\label{section:GP}
\textbf{Gaussian Process:} A GP is a collection of random variables such that any finite subset of them has a joint Gaussian distribution. GPs provide a stochastic, data-driven, supervised machine learning approach to specify the relations between input and output data sets of unknown functions through Bayesian inference \cite{GP2006}. 

Let the noisy output of an unknown (black box) function 
\begin{equation}
\label{eq:unknown}
\mathbf{y} \!=\! \mathbf{f}(\mathbf{x}) \!+\! \epsilon
\end{equation}
with ${\mathbf{x} \!\in\! \mathbf{X} \!\subseteq\! \mathbb{R}^{n_{\mathbf{x}}}, \ \mathbf{f} \!:\! \mathbf{X} \!\to\! \mathbb{R}}$, be perturbed by an i.i.d zero-mean Gaussian noise ${\epsilon \!\sim\! \mathcal{N} (0, \sigma^2_{\epsilon})}$. We assume that function $\mathbf{f}$ is distributed as a GP: ${\mathbf{f}(\mathbf{x}) \!\sim\! \mathbf{GP} (\mu(\mathbf{x}), \mathbf{k}(\mathbf{x}, \mathbf{x}'))}$, with a \textit{mean function} ${\mu \!:\! \mathbf{X} \!\to\! \mathbb{R}}$ and a \textit{covariance/kernel function} ${\mathbf{k} \!:\! \mathbf{X} \!\times\! \mathbf{X} \!\to\! \mathbb{R}_0^+}$, for any ${\mathbf{x}, \mathbf{x}' \!\in\! \mathbf{X}}$. Then, the GP is fully specified as
\begin{eqnarray*}
&\mu(\mathbf{x})  \triangleq \mathop{{}\mathbb{E}} \left[ \mathbf{f}(x) \right],\\
&\mathbf{k}(\mathbf{x}, \mathbf{x}')  \triangleq \mathop{{}\mathbb{E}} \left[ (\mathbf{f}(\mathbf{x}) - \mu(\mathbf{x}) ) (\mathbf{f}(\mathbf{x}') - \mu(\mathbf{x}') ) \right].
\end{eqnarray*}

Given a collection of training data ${\mathcal{D}_n \!\triangleq\! \left\{ \mathbf{x}_i, \mathbf{y}_i \right\}_{i \!=\! 1 }^n}$ with inputs ${\mathcal{X} \!\triangleq\! \left\{ \mathbf{x}_i \right\}_{i \!=\! 1 }^n}$ and outputs ${\mathcal{Y} \!\triangleq\! \left\{ \mathbf{y}_i \!=\! \mathbf{f}(\mathbf{x}_i) \!+\! \epsilon \right\}_{i \!=\! 1 }^n}$, \textit{GP regression} predicts the output based on an input test point ${ \mathbf{x}^* \!\in\! \mathbf{X}}$. The \textit{posterior} of the GP conditioned on the observations, ${\mathbf{f}(\mathbf{x}^*) \big | \mathbf{x}^*, \!\mathcal{D}_n}$, is also a Gaussian distribution with the \textit{posterior mean} $\mu_*$ and the \textit{posterior variance} $\sigma_*^2$ given by
\begin{align}
\begin{split}
\label{eq:update}
 &\!\!\! \mu_*(\mathbf{x}^*)  \!=\! \mu(\mathbf{x}^*) \!+\! \mathbf{K}(\mathbf{x}^*, \!\mathcal{X}) \big(\Tilde{\mathbf{K}}(\mathcal{X}, \mathcal{X}) \big)^{-1} (\mathcal{Y} \!-\! \mu(\mathcal{X})), \\
  &\!\!\! \sigma_*^2(\mathbf{x}^*)  \!=\! \mathbf{k}(\mathbf{x}^*, \mathbf{x}^*) \!-\! \mathbf{K}(\mathbf{x}^*, \!\mathcal{X}) \big(\Tilde{\mathbf{K}}(\mathcal{X}, \mathcal{X}) \big)^{-1} \mathbf{K}(\mathcal{X}, \mathbf{x}^*),
\end{split}
\end{align}
where ${\Tilde{\mathbf{K}}(\mathcal{X}, \mathcal{X}) \!\triangleq\! \mathbf{K}(\mathcal{X}, \mathcal{X}) \!+\! \sigma^2_{\epsilon} \mathbf{I}}$, and ${\mathbf{K}(\mathbf{x}^*, \mathcal{X}) \!\in\! \mathbb{R}^{1 \times n}}$, ${\mathbf{K}(\mathcal{X}, \mathcal{X}) \!\in\! \mathbb{R}^{n \times n}}$, ${\mathbf{K}(\mathcal{X}, \mathbf{x}^*) \!\in\! \mathbb{R}^{n \times 1}}$ are the \textit{covariance matrices} for the set of points, which measure the correlations between the inputs with the components for all ${i, j \!\leq\! n}$: ${\left [ \mathbf{K}(\mathbf{x}^*, \mathcal{X}) \right ]_{1,j} \!=\! \mathbf{k}(\mathbf{x}^*, \mathbf{x}_j\!)}$, ${\left [ \mathbf{K}(\mathcal{X}, \mathcal{X}) \right ]_{i,j} \!=\! \mathbf{k}(\mathbf{x}_i, \!\mathbf{x}_j)}$, ${\left [ \mathbf{K}(\mathcal{X}, \mathbf{x}^*) \right ]_{j,1} \!=\! \mathbf{k}(\mathbf{x}_j, \mathbf{x}^*)}$.

\textbf{Kernels:} Choosing an appropriate kernel function $\mathbf{k}$ is vital for the successful use of GPs since the kernel encodes prior beliefs about the unknown function. To ensure a valid GP distribution, the $\mathbf{k}$ should be a \textit{valid kernel function}. 
\begin{definition} 
\label{def:validkernel}
Let ${\mathbf{X} \!\subseteq\! \mathbb{R}^{n_{\mathbf{x}}}}$ be a nonempty set. A {\em kernel function} ${\mathbf{k} \!:\! \mathbf{X} \!\times\! \mathbf{X} \!\to\! \mathbb{R}}$ is a \textit{positive definite (or valid) kernel} if it is symmetric, ${\mathbf{k}(\mathbf{x}_i, \mathbf{x}_j) \!=\! \mathbf{k}(\mathbf{x}_j, \mathbf{x}_i), \forall \mathbf{x}_i, \mathbf{x}_j \!\in\! \mathbf{X}}$, and 
 \begin{equation*}
     \sum_{i = 1}^n \sum_{j = 1}^n {a_i\!~a_j\!~\mathbf{k}(\mathbf{x}_i, \!\mathbf{x}_j)} = \Bar{a}^\top \Tilde{\mathbf{K}}(\mathcal{X}, \mathcal{X}) \Bar{a} > 0, 
 \end{equation*}
holds for an arbitrary $ {n \!\in\! \mathbb{N}}$, for any set of weighting constants ${\Bar{a} \!\triangleq\! [a_1, \ldots, a_n]^\top}$ with ${a_i \!\in\! \mathbb{R} \!\setminus\! \{0\} }$, and ${x_i \!\in\! \mathbf{X} }$. 
\end{definition}

One frequently used kernel function on a Euclidean space $\mathbb{R}^{n_{\mathbf{x}}}$ is the \textit{squared-exponential} (SE) kernel given by
\begin{equation*}
\mathbf{k}_{SE}(\mathbf{x}_i, \mathbf{x}_j) = \sigma^2_{f} \exp \bigg (  \dfrac{-\|\mathbf{x}_i - \mathbf{x}_j\|^2}{2 \beta^2}  \bigg ) + \sigma^2_{n} \mu(\mathbf{x}_i, \mathbf{x}_j) ,
\end{equation*}
where ${\lambda_h \!\triangleq\! [\beta~\sigma_{f}~\sigma_{n} ]^\top}$ are the tunable hyperparameters and $\mu$ is the Kronecker delta function. We remark that the SE kernel is a valid kernel for the translational motion ${\mathbf{p}}$ in ${\mathbb{R}^3}$.

\textbf{Geometry-aware Kernels on ${\mathbb{S}^3 \!\times\! \mathbb{R}^3}$:} The special Euclidean group ${\mathrm{SE(3)}}$ (or its quaternion representation ${\mathbb{S}^3 \!\times\! \mathbb{R}^3}$) is a non-Euclidean space. Therefore, to encode our priors about underlying functions for GP, we need a valid and geometry-aware kernel function such that ${\mathbf{k}\!:\! \left ( \mathbb{S}^3 \!\times\! \mathbb{R}^3 \right )  \!\times\!  \left ( \mathbb{S}^3 \!\times\! \mathbb{R}^3 \right ) \!\to\! \mathbb{R}_0^+}$.

In order to compare the similarity between two given poses, ${\mathbf{x}_i \!=\!  [\mathbf{q}_i~\mathbf{p}_i]^\top}$ and ${\mathbf{x}_j \!=\!  [\mathbf{q}_j~\mathbf{p}_j]^\top}$, it is necessary to establish a distance metric for $\mathrm{SE(3)}$. We define a distance metric: ${d_{\mathrm{SE(3)}} \!:\! \left ( \mathbb{S}^3 \!\times\! \mathbb{R}^3 \right ) \!\times\!  \left ( \mathbb{S}^3 \!\times\! \mathbb{R}^3 \right ) \!\to\! \mathbb{R}_0^+}$, given by
\begin{equation*}
d_{\mathrm{SE(3)}}(\mathbf{x}_i, \mathbf{x}_j) \!=\! \big \| \left [ \gamma_1 \left \|  \mathbf{p}_i - \mathbf{p}_j \right \|~ \gamma_2 d_{\mathbb{S}^3}(\mathbf{q}_i, \mathbf{q}_j)   \right ] \big\|,
\end{equation*}
where scalars ${\gamma_1 \!>\! 0}$ and ${\gamma_2 \!>\! 0}$ balance the differences between translation and rotation measurements, and ${\gamma_1 \!+\! \gamma_2 \!=\! 1}$. 

Consider an SE kernel with the distance metric ${d_{\mathrm{SE(3)}}}$, ${\mathbf{k}: \left ( \mathbb{S}^3 \!\times\! \mathbb{R}^3 \right )  \!\times\!  \left ( \mathbb{S}^3 \!\times\! \mathbb{R}^3 \right ) \!\to\! \mathbb{R}_0^+}$:  
\begin{equation}
\label{eq:SEK_SE3}
\mathbf{k}(\mathbf{x}_i, \mathbf{x}_j) = \sigma^2_{f} \exp \left (  \dfrac{-d^2_{\mathrm{SE(3)}}(\mathbf{x}_i, \mathbf{x}_j) }{2 \beta^2}  \right ) ,
\end{equation}
where ${\beta \!>\! 0,\!~\sigma_{f} \!>\! 0 }$ are tunable hyperparameters. Similar kernels are proposed in \cite{lang2018gaussian} and \cite{omainska2023rigid}, and utilized in \cite{dacs2023active}, but (\ref{eq:SEK_SE3}) is not a valid kernel on $\mathrm{SE(3)}$ with respect to Definition~\ref{def:validkernel}. To prove this statement, consider the following example: 
\begin{example}
Consider four identical positions $\mathbf{p}_1 \!=\! \mathbf{p}_2 \!=\! \mathbf{p}_3 \!=\!\mathbf{p}_4 \!=\! \bf 0$, and four different unit quaternions: $\mathbf{q}_1 \!=\!  [1~0~0~0 ]^\top$, $\mathbf{q}_2 \!=\! [0~ 1~ 0~ 0]^\top$, $\mathbf{q}_3 \!=\! [1/\sqrt{2}~ 1/\sqrt{2}~ 0~ 0]^\top$, and $\mathbf{q}_4 \!=\! [1/\sqrt{2}~ 0~ 1/\sqrt{2}~ 0]^\top$. Choose $\beta \!=\! 12$, $\gamma_1 \!=\! 0.1$, and $\gamma_2 \!=\! 0.9$.  The eigenvalues of $\mathbf{K}$ are $(-0.0001,~0.0083,~0.0355,~3.9561)$: the negative eigenvalues imply that $\mathbf{K}$ is not positive definite, and thus not valid. 

On the other hand, if we set ${\beta \!=\! 1}$, the resulting values for ${\eig(\mathbf{K})}$ are ${(0.4725, 0.6940, 1.1404, 1.6929)}$, which indicates a valid kernel. This result mirrors the main argument in \cite{feragen2016open}: there exists a ${\beta \!<\! \beta_{max}}$ that leads to a positive definite kernel (see  Section~4 in \cite{jaquier2020bayesian}). However, this na\"ive statement is not generalizable, as explained in \cite{jayasumana2015kernel, feragen2015geodesic, borovitskiy2020matern}.
\end{example}
\textbf{Squared-Exponential Kernel on ${\mathbb{S}^3 \!\times\! \mathbb{R}^3}$:} Valid SE and Mat\'ern kernel examples on Riemannian manifolds, based on the solutions to stochastic partial differential equations, were presented in \cite{borovitskiy2020matern, jaquier2022geometry, kim2024optimization, fichera2024implicit}. A package that provides these kernels can be found in \cite{mostowsky}. On the sphere $\mathbb{S}^{3}$ a SE kernel is given in \cite{borovitskiy2020matern} (Example 9, Eq. 72) as
\begin{eqnarray*}
&\mathbf{k}_{\mathbb{S}^3}(\mathbf{q}_1, \mathbf{q}_2) \!=\! \\ &\!\!\!\!\!\frac{\sigma^2}{C_\infty} \sum_{n=0}^{\infty} c_{n,3} ~\! \mathcal{C}_n^{(1)} ( \cos(d_{\mathbb{S}^3}(\mathbf{q}_1, \mathbf{q}_2)) ) \exp{\Big(\!\!-\frac{\kappa^2}{2} n(n \!+\! 2)\!\Big)} ,
\end{eqnarray*}
where ${c_{n,3}}$ are constants, $\mathcal{C}_n^{(1)}$ are the Gegenbauer polynomials, ${\kappa \!>\! 0}$ is the length scale, ${\sigma \!>\! 0}$ is a hyperparameter to regulate the variability of the GP, and $C_\infty$ is a normalizing constant that guarantees ${\mathbf{k}_{\mathbb{S}^3}(\mathbf{q}, \mathbf{q}) \!=\! \sigma^2}$, see \cite{borovitskiy2020matern}. 

Note that while ${\mathbb{S}^3 }$ is a compact Riemannian manifold,  ${\mathbb{R}^3}$ is non-compact; thus, ${\mathbb{S}^3 \!\times\! \mathbb{R}^3}$ is non-compact. We define a product kernel on this non-compact space as the multiplication of the squared-exponential kernel for the position components and $\mathbf{k}_{\mathbb{S}^3}$ for the rotational components:
\begin{equation*}
\mathbf{k}_{\mathbb{S}^3 \!\times\! \mathbb{R}^3}(\mathbf{x}_i, \mathbf{x}_j) \triangleq \sigma^2_{s} \mathbf{k}_{\mathbb{S}^3}(\mathbf{q}_i, \mathbf{q}_j) \!~ \mathbf{k}_{SE}(\mathbf{p}_i, \mathbf{p}_j)  ,
\end{equation*}
which is a valid kernel on ${\mathbb{S}^3 \!\times\! \mathbb{R}^3}$, with ${\sigma \!>\! 0}$, since both $\mathbf{k}_{SE}$ and $\mathbf{k}_{\mathbb{S}^3}$ are positive definite. 

\textbf{Geometry-aware Bayesian Optimization:}
Bayesian optimization algorithms search for optimal input values over an unknown objective function through iterative evaluations. One first creates a probabilistic model of the function based on previous evaluations and then uses this model to select candidate optimal points for evaluation. In this study, we consider a black box function with the form: 
\begin{equation}
\label{eq:calib_cost}
\!\!\!\mathbf{f}(\mathbf{p}, \mathbf{q}) \!=\! - \left( \!\! \alpha_1 \frac{\mathbf{f}_\mathbf{p}(\mathbf{p})}{\sup_{\mathbf{p} \in \mathbb{R}^3}{|\mathbf{f}_\mathbf{p}}(\mathbf{p})|} \!+\! \alpha_2 \frac{\mathbf{f}_\mathbf{q}(\mathbf{q})}{\sup_{\mathbf{q} \in \mathbb{S}^3}{|\mathbf{f}_\mathbf{q}}(\mathbf{q})|} \! \right),
\end{equation}
where ${\mathbf{f}  \!:\! \mathbb{S}^3 \!\times\! \mathbb{R}^3  \!\to\! \mathbb{R}}$, ${\mathbf{f}_\mathbf{p} \!:\! \mathbb{R}^3  \!\to\! \mathbb{R}}$, ${\mathbf{f}_\mathbf{q} \!:\! \mathbb{S}^3  \!\to\! \mathbb{R}}$, and the terms ${ \sup_{\mathbf{p} \in \mathbb{R}^3}{|\mathbf{f}_\mathbf{p}}(\mathbf{p}) |}$ and ${\sup_{\mathbf{q} \in \mathbb{S}^3 }{|\mathbf{f}_\mathbf{q}}(\mathbf{q}) |}$ normalize a function of position and an orientation to the range ${[-1, 1]}$. The weights ${\alpha_1 \!>\! 0}$ and ${\alpha_2 \!>\! 0}$, with ${\alpha_1 \!+\! \alpha_2 \!=\! 1}$, balance the relative importance of the position and orientation errors. 

In our calibration process, the functional form of $\mathbf{f}$ in (\ref{eq:calib_cost}) is unknown, but it is assumed to satisfy \eqref{eq:unknown} and \eqref{eq:calib_cost}. Pointwise values of these functions can be sampled via noisy measurements. In our approach, Bayesian optimization utilizes these measurements to solve the following optimization problem: 
\begin{equation*}
    {\mathbf{p}^*, \mathbf{q}^*= \ } 
\argmax_{(\mathbf{\mathbf{q}, \mathbf{p}}) \in (\mathbb{S}^3 \!\times\! \mathbb{R}^3)}~  \mathbf{f}(\mathbf{\mathbf{p}, \mathbf{q}}) ,
\end{equation*}
by representing $\mathbf{f}$ with a GP: ${\mathbf{f}(\mathbf{x}) \!\sim\! \mathbf{GP}  (\mu(\mathbf{x}), \mathbf{k}(\mathbf{x}, \mathbf{x}'))}$ for ${\mathbf{x} \!=\! (\mathbf{\mathbf{p}, \mathbf{q}})}$, ${\mathbf{x}' \!=\! (\mathbf{\mathbf{p}', \mathbf{q}'})}$, ${\mathbf{x}, \mathbf{x}' \!\in\! \mathbb{S}^3 \!\times\! \mathbb{R}^3}$. This approach offers the significant advantage that predictive uncertainty quantification can guide the trade-off between exploration and exploitation. This balance is achieved by using a well-designed decision rule to choose optimal actions $\mathbf{x}^*$.   Careful algorithm design can produce sample-efficient algorithms, and therefore rapid in-the-field recalibration of a manipulator.    

In Bayesian optimization one refines a belief about $\mathbf{f}$ at each observation through a Bayesian posterior update. A \textit{utility function} (\textit{acquisition function}) ${V_k \!:\!  \left ( \mathbb{S}^3 \!\times\! \mathbb{R}^3 \right ) \!\to\! \mathbb{R}}$ guides the search for the optimal point. This function evaluates the usefulness of candidate points for the next function evaluation. The utility function is maximized, based on the available data, to select the next query point.  Finally, after a certain number of queries, the algorithm recommends the point is the best estimate of the optimum. In this paper, we use the following Gaussian Process-\textit{upper confidence bound} (GP-UCB) decision algorithm \cite{srinivas2009} to determine the next sampling point based on currently available measurements:
\begin{equation*}
{\mathbf{x}_{k}= \ } 
\argmax_{\mathbf{x} \in (\mathbb{S}^3 \!\times\! \mathbb{R}^3)}~  \big (  \mu_{k-1}(\mathbf{x}) + \sqrt{\beta_k}\!~ \sigma_{k-1}(\mathbf{x}) \triangleq V_k(\mathbf{x}) \big ) ,
\end{equation*}
where ${\beta_k \!>\! 0 }$ is an iteration-varying parameter that weights the uncertainty in the selection of the next sampling point.

\section{Experimental Design for Calibration}
\label{sec:main}
This section presents our geometry-aware Bayesian optimization framework for the experiment design problem of online kinematic calibration. Our goal is to search for a sequence of experimental measurement poses $\mathcal{D}_n \!\triangleq\! \left\{ \mathbf{x}_i, \mathbf{y}_i \right\}_{i \!=\! 1 }^n$, ${\left\{\mathbf{x}_i \!\in\! (\mathbb{S}^3 \!\times\! \mathbb{R}^3)\right\}_{i \!=\! 1 }^n \!\triangleq\! \mathcal{I}_n}$, ${\mathbf{y}_i \!\in\! \mathbb{R}}$ that quickly recalibrate the arm. We remark that since the joint variables are assumed to be uncertain after a miscalibration is detected, we need another input argument for the black box objective function. Therefore, we seek optimal end effector configurations instead of optimal joint variables. 

Assume that the forward kinematics errors can be completely captured by time-invariant DH parameter errors ${\Bar{\delta} \!\triangleq\! [\delta \Bar{\phi}\ \delta \Bar{\alpha}\  \delta \Bar{a}\  \delta \Bar{d}]^\top \!\in\! \mathbb{R}^{4n_j}}$.  For known DH parameters and given joint variables, the value of the \textit{computed end effector pose} is denoted as ${ [\tilde{\mathbf{q}}~\tilde{\mathbf{p}}]^\top \!=\!\mathbf{f}_{\uptheta \to \mathbf{T}}(\Psi) }$, where ${\Psi \!\triangleq\! [\Bar{\phi}~\Bar{\alpha}~\Bar{a}~\Bar{d}]^\top \!\in\! \mathbb{R}^{4n_j}}$ are the \textit{nominal} DH parameters.  However, under parameter errors $\Bar{\delta}$, the \textit{measured (actual) end effector pose} value, ${ [{\mathbf{q}}~{\mathbf{p}}]^\top \!=\!\mathbf{f}_{\uptheta \to \mathbf{T}}(\Psi \!+\!\Bar{\delta} ) }$, might be different than the computed ones.
We represent the forward kinematic error as:
\begin{equation*}
  [{\mathbf{q}}~{\mathbf{p}}]^\top - [\tilde{\mathbf{q}}~\tilde{\mathbf{p}}]^\top = \mathbf{f}_{\uptheta \to \mathbf{T}}(\Psi +\Bar{\delta} ) \!-\! \mathbf{f}_{\uptheta \to \mathbf{T}}(\Psi) \triangleq \Delta({\mathbf{q}}~{\mathbf{p}}),
\end{equation*}
where ${\Delta}$ maps ${\left ( \mathbb{S}^3 \!\times\! \mathbb{R}^3 \right )}$ to ${  \left ( \mathbb{S}^3 \!\times\! \mathbb{R}^3 \right )}$. However, in order to utilize Bayesian optimization, we need a function that maps ${\left ( \mathbb{S}^3 \!\times\! \mathbb{R}^3 \right )}$ to ${\mathbb{R}}$ as given in \eqref{eq:calib_cost}. In this study, we use 
\begin{equation}
\label{eq:objective}
\mathbf{f}_\mathbf{p}(\mathbf{p}) \!=\!   \|  \mathbf{p} - \tilde{\mathbf{p}} \|, \quad  \mathbf{f}_\mathbf{q}(\mathbf{q}) \!=\! d_{\mathbb{S}^3}(\mathbf{q}, \tilde{\mathbf{q}}),
\end{equation} in \eqref{eq:calib_cost} to represent the unknown cost function.

In the implementation of our approach, we use vision-guided manipulation techniques similar to those used in planetary robotics \cite{bajracharya2007, robinson2007, nickels2010} position control frameworks to navigate the robotic arm to specific test locations using a depth camera system. This visual servoing based approach enables accurate end-effector placement despite miscalibration. We employ fiducial markers on the end-effector for high precision, and a camera detects the marker to correct the arm poses (see Section~\ref{sec:hardware}).

Pseudo-code of the proposed Bayesian optimal experimental design for the kinematic calibration is presented in Algorithm~\ref{alg1}. This algorithm implements an online geometry-aware Bayesian optimization to select the sequence of end-effector test locations that will quickly and best recalibrate the manipulator. It leverages the geometry of the special Euclidean group to optimize robot pose selection. The algorithm iteratively selects optimal end-effector poses by maximizing an acquisition function based on the GP model of the kinematic error. For each iteration, the robot is positioned to the chosen pose, and joint variables are measured. The algorithm then computes the expected pose using nominal DH parameters and compares it to the actual measured pose. This comparison yields position and orientation errors, which are used to update the objective function. The GP model is subsequently updated using these new noisy observations, refining the understanding of the kinematic error space. This process continues for a specified number of iterations, resulting in a set of optimally selected poses for kinematic calibration. The algorithm's strength lies in its ability to balance exploration and exploitation of the error space, guided by the GP model's mean and variance predictions.
\begin{algorithm}[tb]
 \caption{Geometry-aware Bayesian optimization for  kinematic calibration experimental design}
 \label{alg1}
 \begin{algorithmic}
 \renewcommand{\algorithmicrequire}{\textbf{Input:}}
 \renewcommand{\algorithmicensure}{\textbf{Output:}}
 \REQUIRE Input set $\mathbf{X}$: ${\mathbf{X} \!\subset\!  (\mathbb{S}^3 \!\times\! \mathbb{R}^3)}$ \\~~~
 Nominal DH parameters ${\Psi}$ \\~~~
 Forward kinematics ${\mathbf{f}_{\uptheta \to \mathbf{T}}(\Psi)}$ \\~~~
 Kernel function $\mathbf{k}_{\mathbb{S}^3 \!\times\! \mathbb{R}^3}$ \\~~~
 GP prior $\mu(\mathbf{x}),~ \sigma^2(\mathbf{x})$ \\~~~
 Parameter $\beta_k$ \\~~~
 Weights ${\alpha_1}$, ${\alpha_2}$, ${\gamma_1}$, ${\gamma_2 \!>\! 0}$
 \ENSURE  ${\mathcal{I}_n \!\triangleq\! \left\{ \mathbf{x}_i \right\}_{i \!=\! 1 }^n}$, ${\mathbf{x}_i \!\in\! (\mathbb{S}^3 \!\times\! \mathbb{R}^3)}$, ${\mathcal{F}_n \!\triangleq\! \left\{ \tilde{\mathbf{x}}_i \right\}_{i \!=\! 1 }^n}$
  \FOR {$i = 1, 2, \ldots, n$,}
  \STATE Choose $ {\mathbf{x}_i^*= \ }
\argmax_{\mathbf{x} \in \mathbb{S}^3 \!\times\! \mathbb{R}^3} ~ V_k(\mathbf{x})$
  \STATE Orient the end effector to the goal pose $\mathbf{x}_i^*$
  \STATE Measure the joint variables $\uptheta_i$
  \STATE Compute the pose ${\tilde{\mathbf{x}}_i = [\tilde{\mathbf{q}}_i~\tilde{\mathbf{p}}_i]^\top }$ with $\uptheta_i$: ${\mathcal{F}_n}$
  \STATE Obtain the unit quaternion with the mapping $\mathbf{f}_{  \mathbf{R} \to \mathbf{q}}$
  \STATE Compute the values of $\mathbf{f}_\mathbf{p}$ and $\mathbf{f}_\mathbf{q}$ in \eqref{eq:objective}
  \STATE Compute the value of the objective function $\mathbf{f}$ in \eqref{eq:calib_cost}
  \STATE Update $\mu_{i}$ and $\sigma_{i}$ via Equation \eqref{eq:update}
  \ENDFOR
 \end{algorithmic} 
 \end{algorithm}

\subsection{Optimization-Based Kinematic Calibration}
\label{sec:para}
Data from Algorithm~\ref{alg1} drives the kinematic calibration process, which minimizes the discrepancy between the observed end-effector data and the calculated forward kinematics data:
\begin{equation}
\label{eq:OptMea}
 \Delta_n = [\mathcal{I}_n]^\top - [\mathcal{F}_n]^\top \in \mathbb{R}^{7n},
\end{equation}
where $n$ is the number of measurements, $\mathcal{I}_n$ represents the measured (actual) pose values, $\mathcal{F}_n$ is the computed end effector pose, and these data sets are defined in Algorithm~\ref{alg1}. 
\begin{remark}
\label{re:quat_data}
We remark that to compute the difference between the unit quaternion components of $\mathcal{I}_n$ and $\mathcal{F}_n$, we use distance formula  \eqref{eq:quad_dist}. We can find the sign of the closest quaternion via the minimum operation: ${ \min \!\left( 2 \cos^{-1} \!\left( \langle \mathbf{q}_1, \mathbf{q}_2 \rangle \right), 2 \cos^{-1} \! \left( \langle \mathbf{q}_1, -\mathbf{q}_2 \rangle \right) \right)}$. Then, we obtain the quaternion difference with respect to the closest one. 
\end{remark}
In order to implement any kinematic calibration method, we need to show that DH parameters are locally identifiable. 
\begin{definition}[Local Identifiability of DH Parameters]
\label{def:dh_identifiability_ball}
The error vector, which affects the output of forward kinematics function ${\mathbf{f}_{\uptheta \to \mathbf{T}}(\Psi)}$, is said to be \textit{locally identifiable} at ${\Psi \!=\! \Psi_0 }$ if there exists a neighborhood ${B(\Psi_0)} \subset \mathbb{R}^{4n_j}$, with ${\Psi_0 \!\in\! B(\Psi_0)}$, such that
\begin{equation*}
\mathbf{f}_{\uptheta \to \mathbf{T}}(\Psi_1) \!=\! \mathbf{f}_{\uptheta \to \mathbf{T}}(\Psi_2) \implies \Psi_1  = \Psi_2, \ \forall \!~\Psi_1, \Psi_2 \!\in\! B(\Psi_0).
\end{equation*}
\end{definition}
The forward kinematics are locally identifiable if there is a locally unique mapping between DH parameters and end-effector locations. At some singular configurations, the parameters are not identifiable. The following theorem establishes a relationship between the rank of the \textit{identification Jacobian matrix}, which is a function of measurement configurations, and the local identifiability of the forward kinematics function. 
\begin{theorem}
\label{the:calib}
Suppose that the identification Jacobian matrix: ${\mathbf{J} \!=\! \frac{\partial \mathbf{f}_{\uptheta \to \mathbf{T}}(\Psi)}{\partial \Psi} }$, with ${\mathbf{J} \!\in\! \mathbb{R}^{7 \times 4n_j}}$, is evaluated for ${n}$ different choices of joint variables ${\Bar{\theta}_1,\ldots,\Bar{\theta}_n}$, such as for the optimal joint variables obtained with Algorithm~\ref{alg1}, which is denoted as ${\mathbf{J}_n \!\in\! \mathbb{R}^{7 n \times 4n_j}}$, such that ${7n \!\geq\! 4n_j}$. Then, ${\mathbf{f}_{\uptheta \to \mathbf{T}}}$ is \textit{locally identifiable} at ${\Psi_1 \!\triangleq\! [\Bar{\theta}_1~\Bar{\alpha}~\Bar{a}~\Bar{d}]^\top }$, $\ldots$, ${\Psi_n \!\triangleq\! [\Bar{\theta}_n~\Bar{\alpha}~\Bar{a}~\Bar{d}]^\top  }$, if ${\rank(\mathbf{J}_n) \!=\! 4n_j}$.
\end{theorem}
\begin{proof}
\label{pr:jacob}
Given ${\mathbf{f}_{\uptheta \to \mathbf{T}}}$ with ${\rank(\mathbf{J}_n) \!=\! 4n_j}$, i.e., ${\mathbf{J}_n}$ having full column rank. By the inverse function theorem, if ${\mathbf{f}_{\uptheta \to \mathbf{T}}(\Psi_i)}$ is continuously differentiable and $\mathbf{J}_n$ has full column rank at ${\Psi_i \!\triangleq\! [\Bar{\theta}_i~\Bar{\alpha}~\Bar{a}~\Bar{d}]^\top }$, then there exists a neighborhood $\mathcal{U}$ of ${\Psi_i}$ where $\mathbf{f}_{\uptheta \to \mathbf{T}}$ is invertible. Since $\mathbf{f}_{\uptheta \to \mathbf{T}}$ is invertible in $\mathcal{U}$, small perturbations in $\Psi$ around $\Psi_i$ lead to distinguishable changes in $\mathbf{f}_{\uptheta \to \mathbf{T}}$, confirming local identifiability at ${\Psi \!=\! \Psi_i}$. Thus, $\mathbf{f}_{\uptheta \to \mathbf{T}}$ is locally identifiable at ${\Psi_i}$ if ${\rank(\mathbf{J}_n) \!=\! 4n_j}$.
\end{proof}
Suppose that a lower bound ${\Bar{\delta}_{lb} \in \mathbb{R}^{4n_j}}$ and an upper bound ${\Bar{\delta}_{ub} \in \mathbb{R}^{4n_j}}$ for the unknown DH parameters are given, and ${\rank(\mathbf{J}_n) \!=\! 4n_j}$. Then, the linearization-based calibration problem can be solved via the following QP:
\begin{align}
{\label{eq:Cal-QP}}
\begin{array}{l}
{\Bar{\delta}^*= \ }
\displaystyle  \argmin_{\Bar{\delta} \in \mathbb{R}^{4n_j}} \ \ \ {\|\Delta_n - \mathbf{J}_n \Bar{\delta}   \|^2}  \\ [1mm]
~~~~~~~~~~~~\textrm{s.t.} ~~~~~~~~~ \Bar{\delta}_{lb} \leq \Bar{\delta}  \leq \Bar{\delta}_{ub} .
\end{array}
\end{align}
Then, we obtain ${\Psi^* \!=\! \Psi \!+\! \Bar{\delta}^*}$, and this QP is iteratively solved until an acceptable error convergence is achieved \cite{cursi2021}.

\begin{remark}
\label{re:jacob}
If some DH parameters are known, the Jacobian matrix $\mathbf{J}$ should be defined by the partial derivatives with respect to the unknown parameters only. Similarly, the QP is constructed and only solved for the unknown parameters. We also remark that Theorem~\ref{the:calib} provides a sufficient condition for demonstrating local identifiability, which in turn directly implies the linear independence of the DH parameters in that specific region. Consequently, only the linearly independent parameters can be estimated via the QP in \eqref{eq:Cal-QP}. Therefore, it is necessary to identify dependent parameters, which are the parameters associated with the linearly dependent columns, before the calibration process. There are several analytical methods to identify linearly dependent DH parameters for serial link manipulators \cite{meggiolaro2000}. We select the number of data points $n$ for Algorithm~\ref{alg1} based on Theorem~\ref{the:calib}, which specifies that the minimum number of data points should satisfy ${7n \!\geq\! 4n_j}$ in addition to the rank condition ${\rank(\mathbf{J}_n) \!=\! 4n_j}$.
\end{remark}

\section{Simulations and Experiments} 
\label{sec:exp}
\subsection{Simulations}
We demonstrate the proposed algorithm on the OWLAT testbed, which simulates the operation of a lander on an icy moon. We compare our results with the random sampling method. A 7-DOF Barrett WAM manipulator simulates robotic sampling operations, see Fig.~\ref{fig:owlat}, that might take place on a future lander mission. The kinematic parameters of the arm can be found in \cite{dacs2023active}. 

In our test, the DH parameter errors of the arm are artificially set to ${\delta \Bar{\phi} \!=\! [\mathbf{0}^\top~1.3] \!~rad}$, ${\delta \Bar{\alpha} \!=\! [0~0.4~\mathbf{0}^\top] \!~rad}$, ${\delta \Bar{a} \!=\! [0~0~0.01~\mathbf{0}^\top] \!~m}$, ${\delta \Bar{d} \!=\! [0~0~0.15~\mathbf{0}^\top] \!~m}$. 
\begin{figure}
\centering
\begin{subfigure}{0.45\columnwidth}
\centering\includegraphics[scale=0.44]{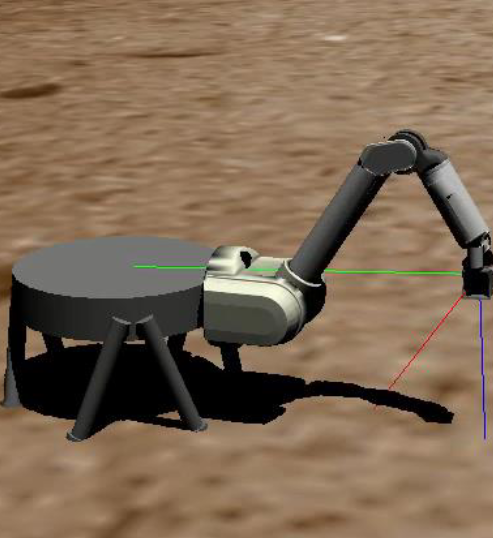}
\end{subfigure} 
\hfill
\begin{subfigure}{0.53\columnwidth}
\centering\includegraphics[scale=0.51]{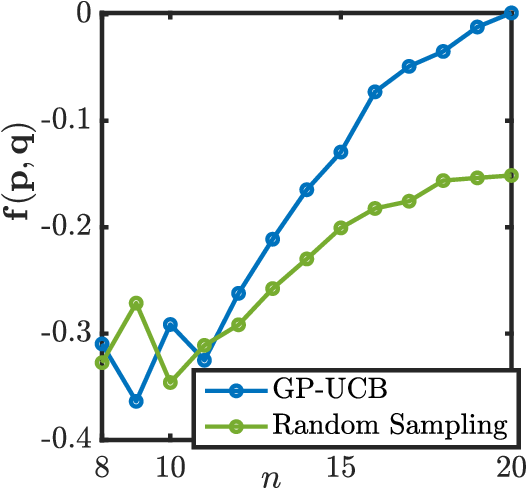}
\end{subfigure} 
\caption{ (Left) The high-fidelity dynamics simulator of OWLAT \cite{tevere2024}. The simulation results (Right). The proposed method demonstrates superior kinematic calibration performance and data efficiency compared to the random sampling method.}
\vskip - 4mm
\label{fig:owlat}
\end{figure}

We run Algorithm~\ref{alg1} for 20 iterations. In each iteration, the end-effector pose values required to compute $\mathbf{f}$ are derived using DH parameters, which are updated by solving the QP problem given in \eqref{eq:Cal-QP} based on the corresponding data points. Fig.~\ref{fig:owlat}-(Right) shows the simulated convergence of the objective function $\mathbf{f}$ versus the Bayesian optimization sample number. The results demonstrate that our framework achieves better calibration performance and data efficiency than random sampling methods for a 7-DOF arm.

\subsection{System Hardware and Experiments}
\label{sec:hardware}
\begin{spacing}{0.97}
We leverage the main results of this paper to demonstrate kinematic calibration experimentally. Our algorithms are first deployed on a simulator to verify viability and safety. We utilize the DARTS (Dynamics And Real-Time Simulation Laboratory) \cite{jain2020darts} simulator to emulate the physical OWLAT testbed at NASA JPL. After verifying feasibility, our algorithms are then tested on the OWLAT testbed housed in NASA JPL’s \cite{tevere2024}. We use the OWLAT Testbed's 7 DoF WAM robotic arm from Barrett Technology and its 2 DoF vision system. The vision system consists of pan and tilt motors with absolute encoders that move an Intel Realsense D415 depth and RGB camera, which we used for localization of the arm. The camera's RGB module provides a 69deg x 42deg  field of view at 1080p resolution and a frame refresh rate of 30 Hz. This capability allows for high-speed localization and adequate scene coverage during our experimental testing, enabling the arm to assume a wide variety of configurations while still remaining in view of the onboard lander cameras. Near the 7th joint, we attached a custom-designed scooping end-effector, which includes two 36h11 AprilTags \cite{apriltags}, which are used to localize the end-effector relative to the arm base, as seen in Fig.~\ref{fig:owlat_tom_edit}. Our algorithms run on an Intel NUC with eight cores and on less than 8GB of RAM. Communication between our algorithm-computer and the WAM arm control-NUC is built on ROS1 over a network.
\end{spacing}

We introduce constant encoder biases into the joint angle measurements actively published by the system, which causes discrepancies in the forward kinematics. We then attempt to perform basic movements such as unstowing the arm. These movements subsequently fail to reach their correct configurations and are ultimately performed incorrectly by the arm. Our fault detection system identifies the faulty joint encoders in question during their movements \cite{touma2023}. It is at this point that our novel recalibration algorithm is activated. A search space is chosen to avoid collisions with the arm itself and the testbed, and infeasible measurements. Then, the prior configurations are chosen to initiate the GP regression. The result of the experiments is presented in Fig.~\ref{fig:main}. This figure illustrates that the proposed method can calibrate the arm with 20 pose measurements with impressive accuracy, thus successfully overcoming the injected biases.
\begin{figure}[t]
    \centering
\includegraphics[width=1\linewidth]{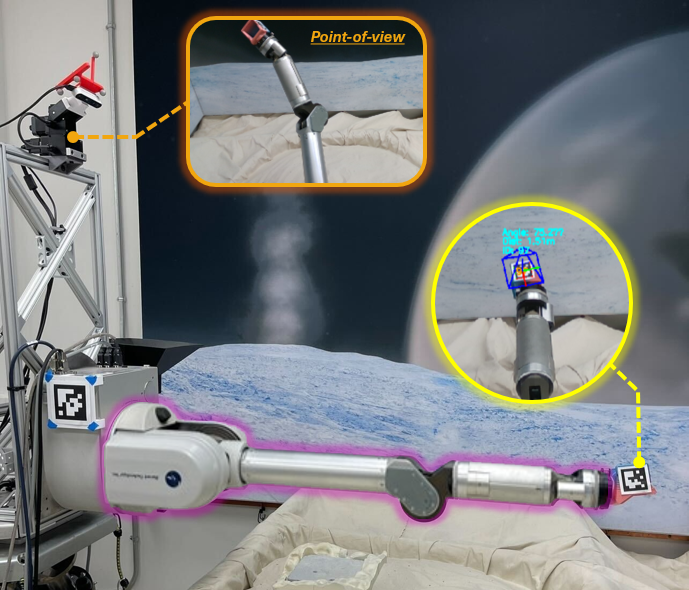}
    \caption{ The OWLAT Testbed at the NASA Jet Propulsion Laboratory (JPL). \textit{(yellow inset)} custom scoop with 36h11 AprilTags installed for precise real-time localization of the end-effector's translation and rotations; \textit{(orange inset)} point-of-view from the lander’s vision system of the WAM arm in one of its other various configurations used for calibration; \textit{(purple inset)} segmented highlight of the WAM arm in one of its calibration poses where it is extended with a visible rotational injected encoder bias in joint-7. 
    }
   \label{fig:owlat_tom_edit}
\end{figure}  
\begin{figure}[t]
    \centering
\includegraphics[width=1\linewidth]{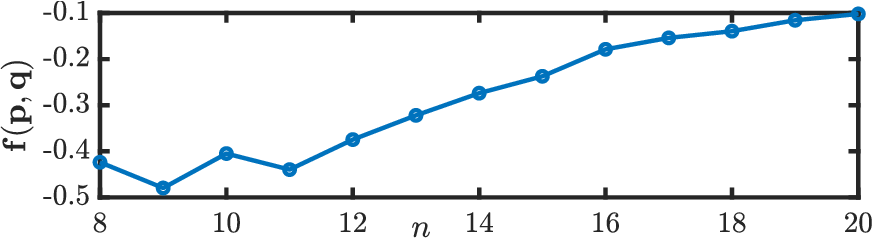}
    \caption{ The value of the objective function $\mathbf{f}$ vs. the number of measurement configurations for the experimental kinematic calibration of NASA JPL's OWLAT. The end-effector pose values are derived using DH parameters updated through solving the QP problem given in \eqref{eq:Cal-QP} with $n$ data points. The proposed method learns the objective function $\mathbf{f}$; therefore, it minimizes the uncertainties in DH parameters via Bayesian optimization. 
    }
   \label{fig:main}
\end{figure}

\section{Conclusion And Future Work} 
\label{sec:conc}
This study developed a Bayesian optimal experimental design method for online kinematic calibration, utilizing geometry-aware valid kernels on ${\mathbb{S}^3 \!\times\! \mathbb{R}^3}$. Instead of framing the design problem in joint space, the method optimizes directly over end-effector poses, which enhances robustness to uncertainties in joint errors and improves overall calibration accuracy. A GP-based learning method was proposed, incorporating a quaternion geodesic distance, the Euclidean distance, and a GP-UCB optimization method (which can also be applied to other rotation-based applications such as visual pursuit control \cite{omainska2023rigid}, or rigid motion learning \cite{lang2018gaussian}). Experiments conducted on a 7DOF robotic arm demonstrated the effectiveness of the method. Future work will focus on improving the smoothness of the objective function $\mathbf{f}$ while considering learning with rotations \cite{rene2024learning} to address the discontinuities in $\mathbf{f}$. 

\clearpage
\begin{spacing}{0.97}
\bibliographystyle{IEEEtran}
\bibliography{Bib/refs}
\end{spacing}

\end{document}